\documentclass{article}

\usepackage[purexy]{qsymbols}
\usepackage{enumerate}
\usepackage{amsmath}
\usepackage{amsfonts}
\usepackage{amssymb}
\usepackage{amsbsy}
\usepackage{amsthm}
\usepackage{isomath}

\SetMathAlphabet{\mathbf}{normal}{OML}{mdput}{b}{n}

\usepackage{dsfont}

\usepackage{algorithm}
\usepackage{algorithmic}
\usepackage{mathrsfs}
\usepackage{paralist}
\usepackage{epsfig}
\usepackage{subfigure}
\usepackage{makeidx} 
\usepackage{color}
\usepackage{pstricks}
\usepackage{pst-node}
\usepackage{multido} 
\usepackage{varioref}
\usepackage{wrapfig} 
\usepackage[sort&compress]{natbib}

\numberwithin{equation}{section} 

\theoremstyle{plain} \newtheorem{remark}{Remark}
\theoremstyle{plain} \newtheorem{definition}{Definition}
\theoremstyle{plain} 
\theoremstyle{plain} 
\theoremstyle{plain} 
\theoremstyle{plain} \newtheorem{theorem}{Theorem}
\theoremstyle{plain} 
\theoremstyle{plain} \newtheorem{lemma}{Lemma}
\theoremstyle{plain}

\newcommand \E {\mathop{\mbox{\ensuremath{\mathbb{E}}}}\nolimits}

\renewcommand \Pr {\mathop{\mbox{\ensuremath{\mathbb{P}}}}\nolimits}

\newcommand{\set}[1]{\left\{\, #1 \,\right\} }
\newcommand{\cset}[2]{\left\{\, #1 ~\middle|~ #2 \,\right\} }

\newcommand{\hP} {\hat{P}}
\newcommand{\hPi} {\hP_\pol}
\newcommand{\Ceps} {\mathcal{N}_\epsilon}

\newcommand{\bP} {P^+}
\newcommand{\bPi} {\bP_{\pol}}

\newcommand\Reals {{\mathds{R}}}

\newcommand \CA {{\mathcal{A}}}
\newcommand \CB {{\mathcal{B}}}

\newcommand \CG {{\mathcal{G}}}

\newcommand \CN {{\mathcal{N}}}

\newcommand \CP {{\mathcal{P}}}

\newcommand \CR {{\mathcal{R}}}
\newcommand \CS {{\mathcal{S}}}
\newcommand \CT {{\mathcal{T}}}

\newcommand \CX {{\mathcal{X}}}

\newcommand \ba {{\mathbf{a}}}

\newcommand \bs {{\mathbf{s}}}

\newcommand \defn {\mathrel{\triangleq}}

\newcommand \argmax{\mathop{\rm arg\,max}}

\newcommand \st {s_t}

\newcommand \at {a_t}

\newcommand \stn {s_{t+1}}

\newcommand \Borel[1]{\CB({#1})}

\DeclareMathAlphabet{\mathpzc}{OT1}{pzc}{m}{it}

\newcommand \Dirichlet {\mathop{\mathpzc{Dirichlet}}\nolimits}

\newcommand \Actions {\CA}
\newcommand \States {\CS}
\newcommand \Trans {\CT}

\newcommand \tran[2] {\tau(#1 \mid #2)}

\newcommand \rew {\rho}
\newcommand \rews {\mathbf{\rho}}
\newcommand \Rews {\CR}

\newcommand \Util {U}

\newcommand \Value {V}

\newcommand \disc {\gamma}
\newcommand \mdp {\mu}
\newcommand \cmp {\nu}
\newcommand \CMPs {\CN}

\newcommand \pol {\pi}
\newcommand \pls[1] {\pol^*(#1)}
\newcommand \Pols {\CP}

\newcommand \regret {\mathcal{L}}

\newcommand \nature {\phi}
\newcommand \prior {\psi}
\newcommand \bel {\xi}

\newcommand \term {q}
\newcommand \smeas {\sigma}

\newcommand \loss {\ell}

\newcommand \EIG[2]{\CG(#1,#2)}

\newcommand\ind[1]{\mathop{\mbox{\ensuremath{\mathbb{I}}}}\left\{#1\right\}}

\newcommand\dd{\,\mathrm{d}}

\def\clap#1{\hbox to 0pt{\hss#1\hss}}

\def\mathrlap{\mathpalette\mathrlapinternal}

\def\mathrlapinternal#1#2{%
           \rlap{$\mathsurround=0pt#1{#2}$}}

\usepackage{psfrag}
\usepackage{pdfsync}

\title{Sparse reward processes}
\author{Christos Dimitrakakis}

\begin{document}
\maketitle
 
\begin{abstract}
  We introduce a class of learning problems where the agent is presented with a series of tasks. Intuitively, if there is a relation among those tasks, then the information gained during execution of one task has value for the execution of another task. Consequently, the agent is intrinsically motivated to explore its environment beyond the degree necessary to solve the current task it has at hand. Thus, in some sense, the model explain the necessity of curiosity.  We develop a decision theoretic setting that generalises standard reinforcement learning tasks and captures this intuition. More precisely, we define a sparse reward process, as a multi-stage stochastic game between a learning agent and an opponent. The agent acts in an unknown environment, according to a utility that is arbitrarily selected by the opponent. Apart from formally describing the setting, we link it to bandit problems, bandits with covariates and factored MDPs. Finally, we examine the behaviour of a number of learning algorithms in such a setting, both experimentally and theoretically.
\end{abstract}

\newcommand{\comment}[1] {}

\section{Introduction}

This paper introduces sparse reward processes. These capture the problem of acting in an unknown environment, with an arbitrary unknown sequence of future objectives. The question is: how to act so as to perform well in the current objective, while at the same time acquiring knowledge that might be useful for future objectives? It is thus analogous to a number of real-world problems with high uncertainty about future tasks, as well as the more philosophical problem of motivating the utility of curiosity in human behaviour.

We formulate this setting in terms of a multi-stage game between a learning agent and an opponent of unknown type. The \emph{agent} acts in an unknown Markovian environment, which is the same in every stage.  At the beginning of each stage, a \emph{payoff} function is chosen by the opponent, which determines the agent's utility for that stage. The agent must act not only so as to maximise expected utility at each stage, but also so that he can be better prepared for whatever payoff function the opponent will select at the next stage.

We call such problems {\em sparse reward processes}, because of two types of sparseness. The first refers to payoff {\em scarcity}: the payoff available at every stage is bounded, while the agent wants to maximise the total payoff across stages. The second refers to the fact that the {\em payoff function} is sparse for an adversarial opponent. We posit that this is a good model of life-long learning in uncertain environments, where while resources must be spent learning about currently important tasks, there is also the need to allocate effort towards learning about aspects of the world which are not relevant at the moment. This is due to the fact that unpredictable future events may lead to a change of priorities for the decision maker. Thus, in some sense, the model ``explains'' the necessity of curiosity.

While our main contribution is the introduction of the problem, we also analyse some basic properties. We show that when the opponent is nature, the problem becomes an unknown MDP. For adversarial opponents, a good strategy for a two-stage version of the game is to maximise the information gain with respect to the MDP model, linking our formulation to exploration heuristics such as compression progress~\citep{schmidhuber1991possibility}, information gain~\citep{Lindley:InformationExperiment} and approximations to the value of information~\citep[Sec. 23.7]{Koller+Friedman:Probabilistic}. For the general adversarial case, we show that either sampling from the posterior~\cite{strens2000bayesian} or confidence-bound based approaches~\citep{mach:Auer+Cesa+Fischer:2002,JMLR:UCRL2} perform well compared to a greedy policy.
However, when the opponent is nature, a greedy policy performs very well, as the payoff stochasticity forces the agent to explore.

The next section introduces the setting and formalises the environment, the payoff, the policy and the complete sparse reward process. Sec.~\ref{sec:optimality} examines the properties that arise for the two opponent types: nature and adversarial. Sec.~\ref{sec:algorithms} briefly explains two algorithms for acting in SRPs, derived from two well-known reinforcement learning exploration algorithms based on confidence bounds and Bayesian sampling respectively. The experimental setup is described in  Sec.~\ref{sec:experiments}, while Sec.~\ref{sec:discussion} concludes the paper with a discussion of related work and links to other problems in reinforcement learning and decision theory.

\section{Setting}
\label{sec:formalisation}
The setting can be formalised as a multi-stage game between the agent and an opponent, on a stochastic environment $\cmp$. At the beginning of the $k$-th stage the opponent chooses a payoff $\rew_k$, which he reveals to the agent, who then selects an arbitrary policy $\pol_k$. It then acts in $\cmp$ using $\pol_k$ until the current stage enters a terminating state. This interaction results in a random sequence of state visits $\bs$, whose utility for the agent is $\rew_k(\bs)$. The agent's goal to minimise the total expected regret $\sum_k \Value^*_k - \Value_k$, where $\Value_k \defn \Value(\rew_k, \pol_k)$  is the expected utility and $\Value^*_k \defn \sup_\pol \Value(\rew_k, \pol)$ is the maximum expected utility for that stage.

If the dynamics are known to the agent, then selecting $\pol_k$ maximising the total expected payoff, only requires playing the optimal strategy for each stage and disregarding the remaining stages. When $\cmp$ is {\em unknown}, however, learning about the environment is important for performing well in the later stages. The setting then becomes an interesting special case of the exploration-exploitation problem.


\subsection{The environment}
\label{sec:environment}
At every stage, the agent is acting within an unknown environment.
We assume that the opponent, has no control over the environment's dynamics and that these are constant throughout all stages.  More specifically, we define the environment to be a controlled Markov process:
\begin{definition}
  A controlled Markov process (CMP) $\cmp$ is a tuple $\cmp = (\States, \Actions, \Trans)$, with state space $\States$, action space $\Actions$, and transition kernel
  \[
  \Trans \defn
  \cset{\tran{\cdot}{s,a}}{s \in \States, a \in \Actions},
  \]
  indexed in $\States \times \Actions$ such that $\tran{\cdot}{s,a}$ is a probability measure\footnote{We assume the measurability of all sets with respect to some appropriate $\sigma$-algebra. This will usually be the Borel algebra $\Borel{X}$ of the set $X$.} on $\States$. If at time $t$ the environment is in state $\st \in \States$ and the agent chooses action $\at \in \Actions$, then the next state $\stn$ is drawn with a probability independent of previous states and actions:
  \begin{equation}
    \label{eq:next-state}
    \Pr_\cmp(s_{t+1} \in S \mid s^t, a^t) = \tran{S}{\st,\at}
    \qquad
    S \subset \States.
  \end{equation}
  Finally, we shall use $\CMPs$ for the class of CMPs.
\end{definition}
In the above, and throughout the text, we use the following conventions.  We employ $\Pr_\cmp$ to denote the probability of events under a process $\cmp$, while we use $s^t \equiv s_1, \ldots, \st$ and $a^t \equiv a_1, \ldots, \at$ to represent sequences of variables. Similarly $\CS^t$ denotes product spaces, and $\CS^* \defn \bigcup_{t=0}^\infty \CS^t$ denotes the set of all sequences of states. Arbitrary-length sequences in $\CS^*$ will be denoted by $\bs$.

Throughout this paper, we assume that the transition kernel is not known to the agent, who must estimate it through interaction. On the other hand, the payoff function, chosen by the opponent, is revealed to the agent at the beginning of each stage.

\subsection{The payoff}
\label{sec:payoff}
At the $k$-th stage, a payoff function $\rew_k : \CS^* \to [0,1]$ is chosen by the opponent.  This encodes how desirable a state sequence is to the agent for the task.  In particular if $\bs, \bs' \in \CS^*$ are two state sequences, then $\bs$ is preferred to $\bs'$ in round $k$ if and only if $\rew_k(\bs) \geq \rew_k(\bs')$. As a simple example, consider a $\rew$ such that, sequences $\bs$ going through a certain state $s^*$ have a payoff of $1$, while the remaining have a payoff of $0$. 

The usual reinforcement learning (RL) setting can be easily mapped to this. Recall that in RL the agent is acting in a Markov decision process~\citep{Puterman:MDP:1994} $\mdp$ (MDP). This is a CMP equipped with a set of distributions
$\cset{R(\cdot \mid s)}{s \in \States}$ on rewards $r_t \in \Reals$. In the {\em infinite-horizon, discounted reward} setting, the utility is defined as the discounted sum of rewards $\sum_t \disc^t r_t$, where $\disc \in [0,1]$ is a discount factor. We can map this to our framework, by setting:  $\rew(s^T) = \sum_{t=1}^T \disc^t \E(r_t \mid s_t) = \sum_{t=1}^T \disc^t \int_{-\infty}^\infty r \dd{R}(r \mid s_t)$ to be the payoff for a state sequence $s^T$. 
While the theoretical development applies to general payoff functions, the experimental results and algorithms use the RL setting.

\subsection{The policy}
\label{sec:policy}
After the payoff $\rew_k$ is revealed to the decision maker, he chooses a policy $\pol_k$, which he uses to interact with the environment.  The CMP $\cmp = (\States, \Actions, \Trans)$ and the payoff function jointly define an MDP, denoted by $\mdp_k = (\States, \Actions, \Trans, \rew_k)$.  The agent's policy $\pol_k$ selects actions with distribution $\pol_k(\at \mid s^t)$, meaning that the policy is not necessarily stationary. Together with the Markov decision process $\mdp_k$, it defines a distribution on the sequence of states, such that:
\begin{align*}
  \Pr_{\mdp_k,\pol_k}(s_{t+1} \in S \mid s^t)
  &=
  \int_\Actions \tran{S}{a,\st} \dd{\pol}(a \mid s^t).
\end{align*}
This interaction results in a sequence of states $\bs$, whose utility to the agent is:
$\Util_k \defn \rew_k(\bs)$, $\bs \in \CS^*$.
Since the sequence of states is stochastic, we set the value of each stage to the {\em expected utility}:
\begin{equation}
  \Value_k \defn V(\rew_k, \pol_k) \defn \E_{\cmp, \pol_k} \Util_k
  = \int \rho_k(\bs) \dd{P}^\pol_\cmp(\bs),
\end{equation}
where ${P}^\pol_\cmp$ is the probability measure on $\CS^*$ resulting from using policy $\pol$ on CMP $\cmp$.
Finally, let us define the oracle policy for stage $k$:
\begin{definition}[Oracle stage policy]
  Given the process $\cmp$ and the payoff $\rew_k$ at stage $k$, the oracle policy is $\pls{\cmp,\rew_k} \defn \argmax_{\pol} \int_{\CS^*} \rew_k(\bs) \dd{P}^\pol_\cmp(\bs)$.
\end{definition}
This policy is normally unattainable by the agent, since $\cmp$ is unknown. The agent's \emph{goal} is to minimise the total expected regret~\footnote{We use a slightly different notion of regret from previous work. Instead of using the total accumulated reward, as in~\citep{JMLR:UCRL2}, we consider the total expected utility across stages. But, if one were to see the payoff obtained at every stage as the reward, the two measures of regret would be equivalent.}
 relative to the oracle:
\begin{equation}
\regret_K \defn \sum_{k=1}^K \Value(\rew_k, \pls{\cmp,\rew_k}) - \Value_k
\label{eq:regret}
\end{equation}

\subsection{Sparse reward processes}
\label{sec:srp}
The complete sparse reward process is a special case of a stochastic game. However, we are particularly interested in processes where only few states  have payoffs. We model this by mapping each payoff function to a {\em finite} measure on $\States^*$. 
\begin{definition}
  A sparse reward process is a multi-stage stochastic game with $K$ stages, where the $k$-th stage is a Markov decision process $\mdp_k  = (\CS, \CA, \CT, \rew_k)$, whose payoff function $\rew_k : \CS^* \to [0,1]$, is revealed to the agent immediately after $k-1$ stage is complete. The agent chooses policy $\pol_k$, with expected utility $\Value_k \defn V(\rew_k, \pol_k)$. The Markov decision process terminates at time $t$ the stage ends, if $s_t$ is a terminal state and with fixed termination probability $\term$ if $s_t$ is not a terminal state.

  The process is called $\smeas$-sparse for a measure $\smeas$ on $\CS^*$ if for every $\rew_k \in \Rews$, the payoff measure $\lambda_k$ on $\CS^*$, defined as
  $\lambda_k(S) = \int_S \rew_k(\bs) \dd{\smeas(\bs)}$, $\forall S \subset \CS^*$, satisfies $\lambda_k(\CS^*) \leq 1$. The agent's goal is to find a sequence $\pol_k$ maximising $\sum_{k=1}^K V_k$.
  \label{def:srp}
\end{definition}
The termination probability $\term$ is equivalent to an infinite horizon discounted reward reinforcement learning problem~\citep[see][]{Puterman:MDP:1994}.  Bounding the total payoff  forces the rewards available at most state sequences to be small (though not necessarily zero). Finally, it ensures that the opponent cannot place arbitrarily large rewards in certain parts of the space, and so cannot make the regret arbitrarily large. Throughout the paper, we take $\smeas$ to be the uniform measure. The construction also enables much of the subsequent  development, through the following lemma:
\begin{lemma}
  Given a payoff function $\rew$ for which there exists a payoff measure $\lambda$ satisfying the conditions of Def.~\ref{def:srp} for some $\smeas$, the utility of any policy $\pol$ on the MDP $\mdp = (\cmp, \rew)$, can be written as:
  \begin{equation}
    \label{eq:sparse-utility}
    \E_{\pol, \mdp} U = \int_{\States^*} p_{\pol,\cmp}(\bs) \dd{\lambda}(\bs),
  \end{equation}
  where $p_{\pol, \cmp}$ is the probability (density) of $\bs$ (with respect to $\smeas$) under the policy $\pol$ and the environment $\cmp$.  We assume that $p_{\pol, \cmp}$ always exists, but is not necessarfily finite.
\end{lemma}
\begin{proof}
  Via change of measure:
  $\E \Util
   = 
    \int_{\mathrlap{\States^*}} \rew(\bs) \dd{P_{\pol, \cmp}} (\bs)
    = 
    \int_{\mathrlap{\States^*}} \rew(\bs) p_{\pol, \cmp} (\bs) \dd{\smeas}(\bs)
    = 
    \int_{\mathrlap{\States^*}} p_{\pol, \cmp} (\bs) \dd{\lambda}(\bs)$.
\end{proof}

\section{Properties}
\label{sec:optimality}
The optimality of an agent policy depends on the assumptions made about the opponent.  In a worst-case setting, it is natural to view each stage as a zero-sum game, where the agent's regret is the opponent's gain. If the opponent is nature, then the sparse reward process can be seen as an MDP. This is also true in the case where we employ a prior over the opponent's type. 

\subsection{When the opponent is nature}
\label{sec:nature}
Consider the case when the opponent selects the payoffs $\rew_k$ by drawing them from some fixed, but unknown distribution with measure $\nature( \cdot \mid \theta)$, parametrised by $\theta \in \Theta$, such that:
$\Pr(\rew_k \in B) = \nature(B \mid \theta)$,
$\forall  k \in \set{1, 2, \ldots, K},~ \forall B \subset \Rews$.
In that case, the Bayes-optimal strategy for the agent is to maintain a belief on $\Theta \times \CMPs$ and solve the problem with backwards induction~\cite{Degroot:OptimalStatisticalDecisions}, if possible. This is because of the following fact:
\begin{theorem}
  When the opponent is Nature, the SRP is an MDP.
\end{theorem}
\begin{proof}
  We prove this by construction. For a set of reward functions
  $\Rews$, the state space of the MDP can be factored into the reward
  function and the state of the dynamics, so $\States = \Rews \times
  \States_0$. If there are $K$ reward functions, we can write the
  state space as $\States = \bigcup_{k=1}^K \States_k$. Let the action
  space be $\Actions$, such and a set of bijections $M_{ij} :
  \States_i \leftrightarrow \States_j$. In addition, for any $i,j$ all
  states $s \in \States_i$, the transition probabilities obey:
  $\Pr(s_{t+1} \mid s_t = s, a_t = a) = \Pr(s_{t+1} \mid s_t =
  M_{i,j}(s), a_t = a)$ and $\Pr(s_{t+1} \in \States_j \mid s_t \in
  \States_i, a_t = a) = q$ if $j \neq i$ and $1 - q$ otherwise.  It is
  easy to verify that this agrees with Def.~\ref{def:srp}.
\end{proof}
Unfortunately, the Bayes-optimal solution is usually intractable~\citep{Gittins:1989,Degroot:OptimalStatisticalDecisions,duff2002olc}. 
\subsection{When the opponent is adversarial}
\label{sec:adversarial}
We look at the problem from the perspective of Bayesian experimental design. In particular, the agent has a belief, expressed as a measure $\bel$ over $\CMPs$. Then, the $\bel$-expected utility of any policy $\pol$ is:
\begin{align}
  \label{eq:subjective-stage-utility}
  \E_{\bel, \pol} \Util
  =
  \int_{\CMPs}
    \int_{\CS^*} \Util(\bs) \dd{P}^\pol_\cmp(\bs)
  \dd{\bel}(\cmp).
\end{align}
Let $P^*_\cmp$ and $P^*_\bel$ be the probability measures on $\CS^*$ arising from the optimal policy given the full CMP $\cmp$ and given a particular belief $\bel$ over CMPs respectively, assuming known payoffs $\rew$.  The opponent can take advantage of the uncertainty and select a payoff function that maximises our loss relative to the optimal policy:
\begin{equation}
  \label{eq:loss}
  \loss_k(\bel, \mdp) \defn \max_\lambda \int_{\CS^*} (P^*_\cmp - P^*_\bel) \dd{\lambda_k}.
\end{equation}
This implies that the opponent should maximise the payoff for sequences with the largest probability gap between the $\cmp$- and $\bel$-optimal policies. To make this non-trivial, we have restricted the payoff functions to $\lambda(\CS^*) \leq 1$. In this case, maximising $\loss$ requires setting $\lambda(B) = 1$ for the set of sequences $B$ with the largest gap, and $0$ everywhere else. This the second type of sparseness that SRPs have.

We now show that in a special two-stage version of our game, a strategy that maximises the expected information gain minimises a bound on our expected loss. First, we recall the definition of~\citet{Lindley:InformationExperiment}: 
\begin{definition}[Expected information gain]
  Assume some prior probability measure $\bel$ on a parameter space $\CMPs$, and a set of experiments $\Pols$, indexing a set of measures $\cset{P^\pol_\cmp}{\cmp \in \CMPs, \pol \in \Pols}$ on $\CX$. The expected information gain of the experiment $\pol$ consisting of drawing an observation $x$ from the unknown $P^\pol_\cmp$ is:
  \begin{equation}
    \label{eq:expected-information-gain}
    \EIG{\pol}{\bel} \defn \int_\CMPs \int_\CX \ln \frac{P^\pol_\cmp(x) }{P^\pol_\bel(x)}\dd{P^\pol_\cmp}(x) \dd{\bel}(\cmp),
  \end{equation}
  where $P^\pol_\bel(x)$ is the marginal $\int_\CMPs P^\pol_\cmp(x) \dd{\bel}(\cmp)$.
  \label{def:expected-information-gain}
\end{definition}
In our case, the parameter space $\CMPs$ is the set of environment dynamics while the observation set $\CX$ is the the set of state-action sequences $(\CS \times \CA)^*$. 
\begin{theorem}
  Consider a {\em two-stage} game, where there for the the first stage,  $\rew_1 = 0$. Then, maximising the expected information gain, in sufficient to minimise the expected regret.
\end{theorem}
\begin{proof}
  Through the definition of the stage $k$ loss \eqref{eq:loss}, and as $\lambda(S^*) \leq 1$:
  \begin{equation}
    \begin{split}
      \loss_k(\bel, \cmp)
      &= \max_\lambda \int_{\mathrlap{\CS^*}} ~(P^*_\cmp - P^*_\bel) \dd{\lambda} 
      \leq \max_\lambda \int_{\mathrlap{\CS^*}} ~|P^*_\cmp - P^*_\bel| \dd{\lambda}
      \\
      & \leq \lambda(S^*) \int_{\mathrlap{\CS^*}} ~|P^*_\cmp - P^*_\bel|
      \dd{\sigma} \leq \|P^*_\cmp - P^*_\bel\|,
    \end{split}
  \end{equation}
  where $\|P\| = \int |P| \dd{\smeas}$ is the $L_1$-norm with respect to $\smeas$.
If our initial belief is $\prior$ and the (random) posterior after the first stage is $\bel$, the expected loss of policy $\pol$ is given by:
$\E_{\prior,\pol} \|P^*_\cmp - P^*_\bel\|$.
  Finally, since for any measures $P, Q$ it holds that $2 \int \ln P/Q
  \dd{P} \leq \|P - Q\|^2$, we have: $\EIG{\pol}{\bel} \geq
  \E_{\bel,\pol} \frac{1}{2}\|P^\pol_\cmp - P^\pol_\bel\|^2_1$. Via
  Jensen's inequality we obtain that $\E_{\bel,\pol} \|P^\pol_\cmp -
  P^\pol_\bel\|_1 \leq \sqrt{2 \EIG{\pol}{\bel}}$.
\end{proof}
Thus, choosing a policy that maximises the expected information gain, minimises the expected worst-case loss at the next stage.  This is in broad agreement with past ideas of relating curiosity to gaining knowledge about the environment (e.g. work such as \citep{schmidhuber1991possibility}). Consequently, pure information-gathering strategies can have good quality guarantees in this two-stage adversarial game.

For more general games, we must employ other strategies, however, as we need to balance information gathering (exploration) with obtaining rewards in the current stage (exploitation).  Unfortunately, even finding the policy that maximises \eqref{eq:expected-information-gain} is as hard as finding the Bayes-optimal policy.  For this reason, in the next section we consider approximate algorithms.

\begin{algorithm}
  \begin{algorithmic}[1]
    \FOR{$k=1, \ldots, K$}
    \STATE Find the largest $\{\Ceps(\hPi) \mid \pol \in \Pols\}$ s.t. 
    $\Pr(\exists \pol : P_\pol \notin \Ceps(\hPi) \leq 1 / k)$, 
    \STATE Select $\pol_k \in \argmax_\pol V^+_\pol(\rew_k)$,  from \eqref{eq:optimistic-value}.
    \STATE Execute $\pol_k$, observe $\bs, \ba$; get payoff $\rew(\bs)$.
    \STATE Update $\{\hPi \mid \pol \in \Pols\}$ from $\bs, \ba$.
    \ENDFOR
  \end{algorithmic}
  \caption{UCSRP: Upper Confidence bound SRP}
  \label{alg:UCSRP}
\end{algorithm}
\begin{algorithm}
  \begin{algorithmic}[1]
    \STATE Set initial beliefs $\bel_1(P_\pol)$.
    \FOR{$k = 1, \ldots, K$}
    \STATE Sample $\hat{\cmp} \sim \bel_k$.
    \STATE Choose $\pol_k = \pls{\hat{\cmp}, \rew_k}$.
    \STATE Execute $\pol_k$, observe $\bs, \ba$; get payoff $\rew(\bs)$.
    \STATE Calculate new posterior $\bel_{k+1}(\cdot) \defn \bel_k(\cdot \mid \bs, \ba)$.
    \ENDFOR
  \end{algorithmic}
  \caption{BTSRP: Bayesian Thompson sampling SRP}
  \label{alg:BTSRP}
\end{algorithm}

\section{Algorithms}
\label{sec:algorithms}
We use two simple algorithms for SRPs, derived from two well-known strategies for exploration in bandit problems and reinforcement learning in general. The first, Upper Confidence bound SRP (UCSRP, Alg.~\ref{alg:UCSRP}) chooses policies based on simple confidence bounds, similarly to UCB~\citep{mach:Auer+Cesa+Fischer:2002} for bandit problems and UCRL~\citep{JMLR:UCRL2} for general reinforcement learning. The second, Bayesian Thompson sampling (BTSRP, Alg.~\ref{alg:BTSRP}), chooses a policy by drawing samples from a posterior distribution, as in~\cite{strens2000bayesian}. To simplify the exposition, we restrict our attention to some arbitrary stage $k$ and consider a setting where we have a finite set of policies $\Pols$. 

\textbf{UCSRP} (Alg.~\ref{alg:UCSRP}) uses confidence regions. An abstract view of the method is the following. For any policy $\pol$, let the empirical measure on $\States^*$ be $\hat{P}_\pol$, and let:
\begin{equation}
  \Ceps(\hPi) \defn \cset{Q}{\|Q - \hPi\| \leq \epsilon}
  \label{eq:confidence}
\end{equation}
be a \emph{confidence region} around the empirical measure, where $\|P\| = \int |P| \dd{\smeas}$ is the $L_1$-norm with respect to $\smeas$. Then we define the \emph{optimistic value}:
\begin{align}
  V^+_\pol \defn \max \cset{\E_Q \rew}{Q \in \Ceps(\hPi)} 
  \label{eq:optimistic-value}
\end{align}
 to be the value within the interval maximising the expected payoff. This can be seen as an optimistic evaluation of the policy $\pi$ that holds with high probability and we choose $\pol_k \in \argmax_\pol V^+_\pol$. For RL problems, there is no need to evaluate all policies. The algorithm can be implemented efficiently via the augmented MDP construction in~\cite{JMLR:UCRL2}.

\textbf{BTSRP} (Alg.~\ref{alg:BTSRP}) draws a candidate CMP $\cmp_k \sim \bel_k$ from the belief at stage $k$, and then calculates the stationary policy that is optimal for $(\cmp_k, \rew_k)$. At the end of the stage, the belief is updated via Bayes's theorem:
\begin{align}
  \bel_{k+1}(B) 
  &=
  \frac{\int_B P_\cmp(\bs \mid \ba) \dd{\bel_k}(\cmp)}
  {\int_\CMPs P_\cmp(\bs \mid \ba) \dd{\bel_k}(\cmp)},
  \label{eq:posterior}
\end{align}
This type of Thompson sampling~\citep{thompson1933lou} performs well in multi-armed bandit problems~\citep{agrawal:thompson}, but its general properties are unknown.

\iftrue
\begin{lemma}
  Consider a payoff function $\rew$ with corresponding payoff measure $\lambda$.  Assume that $\epsilon$ is such that confidence regions hold, i.e. that $P_\pol \in \Ceps(\hPi)$ for all $\pol$. For UCSRP to choose a sub-optimal policy $\pol$, it sufficient that:
  \[
  \E (\rew \mid P_{\pol^*})
  \leq
  \E (\rew \mid P_\pol) + 2 \int c_\pol \dd{\lambda}.
  \]
\end{lemma}
\begin{proof}
  Since UCSRP always chooses $\pol$ maximising $\bPi$, if we choose a sub-optimal $\pol$ then it must hold that $\E (\rew \mid \bP_{\pol^*}) \leq \E ( \rew \mid \bPi)$.  Since the confidence regions hold, $\E (\rew \mid P_{\pol^*}) \leq \E (\rew \mid \bP_{\pol^*})$, $\E (\rew \mid {P_{\pol^*}})\leq \E (\rew \mid {\bP_{\pol^*}})$ and $\E (\rew \mid \hPi) \leq \E (\rew \mid P_\pol + c_\pol)$. Consequently:
  \begin{align*}
    \E (\rew \mid P_{\pol^*})
    &\leq
    \E (\rew \mid {\bPi})
    = 
    \int (\hPi + c_\pol) \dd{\lambda}
    \\
    &\leq
    \int (P_\pol + c_\pol) \dd{\lambda} 
    =
    \E (\rew \mid P_\pol) + 2 \int c_\pol \dd{\lambda}
  \end{align*}
\end{proof}
\begin{theorem}
  Let $c_{\pol, k}$ be the relevant signed measure for policy $\pol$ in stage $k$. Assume that $\exists a, b > 0$ s.t.  $\|c_{\pol,k}\| \leq a n^{-b}_{\pol, k}$, with $n_{\pol_k} = \sum_{i=1}^k \ind{\pol_i = \pol}$. 
\end{theorem}
\begin{proof}
  It will be convenient to use $\rew' \pol = V(\rew, \pol)$ to denote the value of $\pol$ for payoff function $\rew$. This has the usual vector meaning. Let $\rews = (\rew_k)$ be a sequence of payoffs and let $\pol^*_k \in \argmax_{\pol \in \Pols} \rew_k' \pol$ be an optimal policy at stage $k$ in hindsight, and let $\pol_k$ be our actual policy for that stage. Then the regret after $K$ stages, $\regret_K$, is bounded as follows:
  \begin{align*}
  \regret_K
  &\leq \max_\rews \sum_{k=1}^K  \left(\rew_k' \pol^*_k - \rew_k' \pol_k\right)
 \\ &
  = \max_\rews \sum_{k=1}^K \left( \rew_k' \pol^*_k - \rew_k' \sum_{\pol \in \Pols} \pol \ind{\pol_k = \pol} \right)
  \\
  &\leq \sum_{\pol \in \Pols} \sum_{k=1}^K  \max_{\rew_k} \ind{\pol_k = \pol}  \rew_k'(\pol^*_k - \pol)
  \\
  &
  \leq \sum_{\pol \in \Pols} \sum_{k=1}^K 
  \max_{\rew_k} 
  \ind{ \epsilon_{\pol, k} \geq \Delta_{\pol, k}}
  \Delta_{\pol, k}
  \end{align*}
  where $\epsilon_{\pol, k} = 2\|c_{\pol, k}\|_{\lambda_k, 1}$, $\Delta_{\pol, k} = \rew_k'(\pol^*_k - \pol)$.
\end{proof}
The actual shape of the confidence region, for UCSRP, and the belief, for BTSRP, depend on the model we are using. In general, they have the form $c_i = a n_i^{-b}$, where $n_i$ is the number of times the $i$-th policy was chosen and $a > 0$, and $b \geq 1$, but can be tighter if there is an interrelationship between policies. 
\fi

In both cases, a new \emph{stationary} policy is selected at the beginning of each stage. We consider two types of games, for which we employ slightly different versions of the main algorithms. In the first game, each stage terminates after the first action is taken. In the second game, each stage terminates only with constant probability $\term$ at every time-step.

\comment{CD: If we actually could ask somebody to suggest us the oracle policy, then we would not be able to learn anything in the next stage perhaps. So, we need to lower bound the information we gain by actually using the oracle policy.}

\section{Experiments}
\label{sec:experiments}

We consider games having a total of $K$ stages. In each stage, the agent observes an  payoff function of the form $\rew_k(s^T) = \sum_{t=1}^T r(s_t)$ and then selects a policy $\pol_k$.  The environment is Markov, and the stage terminates with fixed probability $q$, known to the agent. 

Confidence intervals for UCSRP can be constructed via the bound of \citet{weissman2003ide} on the $L_1$ norm of deviations of empirical estimates of multinomial distributions. In order to construct an upper confidence bound policy efficiently, we employ the method of UCRL~\citep{JMLR:UCRL2}, 
This solves an augmented MDP where the action space is enlarged to additionally select between high-probability MDPs. This guarantees that the policy acts according to the most optimistic MDP in the high-probability region, as required by UCSRP.

For the BTSRP policy, we maintain a product-Dirichlet distribution~(see for example \cite{Degroot:OptimalStatisticalDecisions}) on the set of multinomial distributions for all state-action pairs and a product of normal-gamma distributions for the rewards. We then draw sample MDPs by drawing parameters from each individual part of the product prior.

\subsection{Opponents}
\label{sec:experiments-opponents}
For reasons of tractability, and better correspondence with the reinforcement larning setting, the opponents we consider consider only additive payoff functions, such that the same reward $r(s)$ is always obtained when visiting state $s$ and the payoff of a sequence of states $s_1, \ldots, s_t$ is simply $\rew(s_1, \ldots, s_t) = \sum_{k=1}^t r(s_t)$. We consider two types of opponents, nature, and a myopic adversary.

\paragraph{Nature.} In this case, the reward functions are sampled uniformly such that $\sum_{s \in S} r(s) = 1$.

\paragraph{Adversarial.} This opponent has knowledge of $\cmp$, and also maintains the empirical estimate $\hat{\cmp}$. Assuming that the agent's estimates must be close to the empirical estimate, the payoff is selected to maximise the stage loss \eqref{eq:loss}. This is a sparse payoff, as explained in Sec.~\ref{sec:adversarial}, based on the empirical estimate rather than the (unknown) agent's belief:
\begin{align}
  \rew_k \in \argmax_\rew \Value(\rew, \pol^*(\cmp, \rew)) - \Value(\rew, \pol^*(\hat{\cmp}, \rew)).
\end{align}

\subsection{Results}
The results are summarised in Fig.~\ref{fig:markov-comparison}. These also include a greedy agent, i.e. the stationary policy which maximises payoff for the current stage in empirical expectation. In both cases, the regret suffered by the greedy agent grows linearly, while that of BTSRP and UCSRP grows slowly for adversarial opponents. However, when the opponent is nature this is no longer the case. This is due to the fact that the distribution of payoffs provides a natural impetus for exploration even for the greedy agent.
\begin{figure*}[htb]
  \centering
  \subfigure[Adversarial, $|\CS| = 4, |\CA|=2, \term=0.5$]{
    \psfrag{greedy}[r][r][1.0][0]{greedy}
    \psfrag{BTSRP}[r][r][1.0][0]{BTSRP}
    \psfrag{UCSRP}[r][r][1.0][0]{UCSRP}
    \psfrag{0}[r][r][1.0][0]{0}
    \psfrag{200}[r][r][1.0][0]{200}
    \psfrag{400}[r][r][1.0][0]{400}
    \psfrag{600}[r][r][1.0][0]{600}
    \psfrag{800}[r][r][1.0][0]{800}
    \psfrag{1000}[r][r][1.0][0]{1000}
    \psfrag{20}[r][r][1.0][0]{20}
    \psfrag{40}[r][r][1.0][0]{40}
    \psfrag{60}[r][r][1.0][0]{60}
    \psfrag{80}[r][r][1.0][0]{80}
    \psfrag{100}[r][r][1.0][0]{100}
    \psfrag{120}[r][r][1.0][0]{120}
    \includegraphics[width=.475\textwidth]{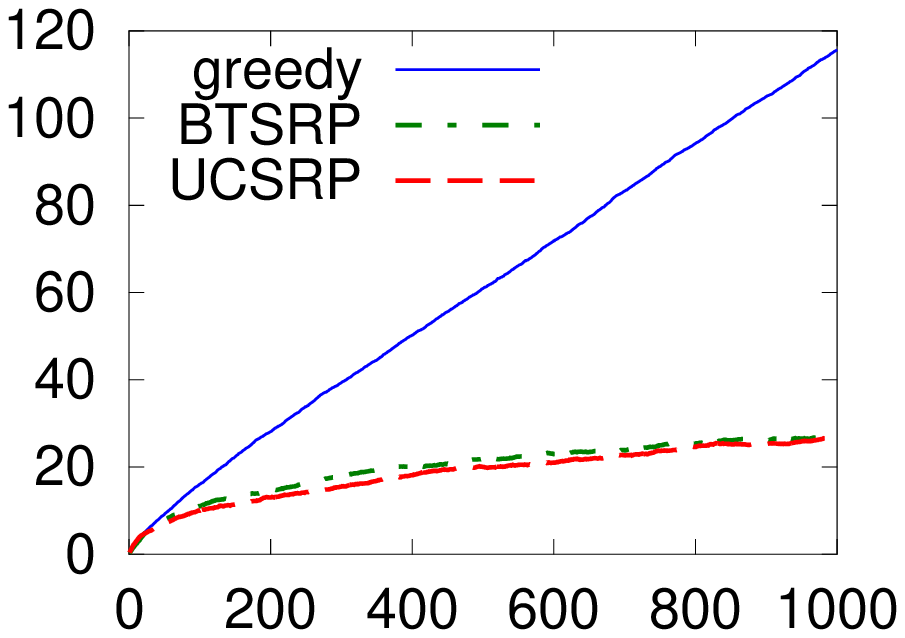}
    \label{fig:four-states}
  }
  \subfigure[Adversarial, $|\CS| = 8, |\CA|=4, \term=0.5$]{
    \psfrag{greedy}[r][r][1.0][0]{greedy}
    \psfrag{BTSRP}[r][r][1.0][0]{BTSRP}
    \psfrag{UCSRP}[r][r][1.0][0]{UCSRP}
    \psfrag{0}[r][r][1.0][0]{0}
    \psfrag{100}[r][r][1.0][0]{200}
    \psfrag{200}[r][r][1.0][0]{200}
    \psfrag{300}[r][r][1.0][0]{300}
    \psfrag{400}[r][r][1.0][0]{400}
    \psfrag{500}[r][r][1.0][0]{500}
    \psfrag{600}[r][r][1.0][0]{600}
    \psfrag{800}[r][r][1.0][0]{800}
    \psfrag{1000}[r][r][1.0][0]{1000}
    \includegraphics[width=.475\textwidth]{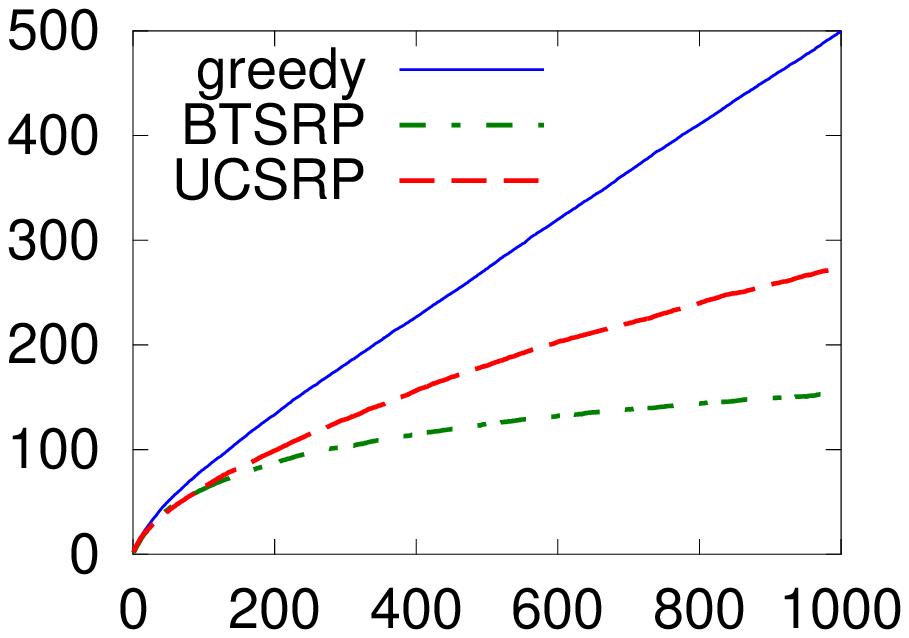}
    \label{fig:eight-states}
  }
  \subfigure[Nature, $|\CS| = 4, |\CA|=2, \term=0.5$]{
    \psfrag{greedy}[r][r][1.0][0]{greedy}
    \psfrag{BTSRP}[r][r][1.0][0]{BTSRP}
    \psfrag{UCSRP}[r][r][1.0][0]{UCSRP}
    \psfrag{0}[r][r][1.0][0]{0}
    \psfrag{200}[r][r][1.0][0]{200}
    \psfrag{400}[r][r][1.0][0]{400}
    \psfrag{600}[r][r][1.0][0]{600}
    \psfrag{800}[r][r][1.0][0]{800}
    \psfrag{1000}[r][r][1.0][0]{1000}
    \psfrag{20}[r][r][1.0][0]{20}
    \psfrag{40}[r][r][1.0][0]{40}
    \psfrag{60}[r][r][1.0][0]{60}
    \psfrag{80}[r][r][1.0][0]{80}
    \psfrag{100}[r][r][1.0][0]{100}
    \psfrag{120}[r][r][1.0][0]{120}
    \includegraphics[width=.475\textwidth]{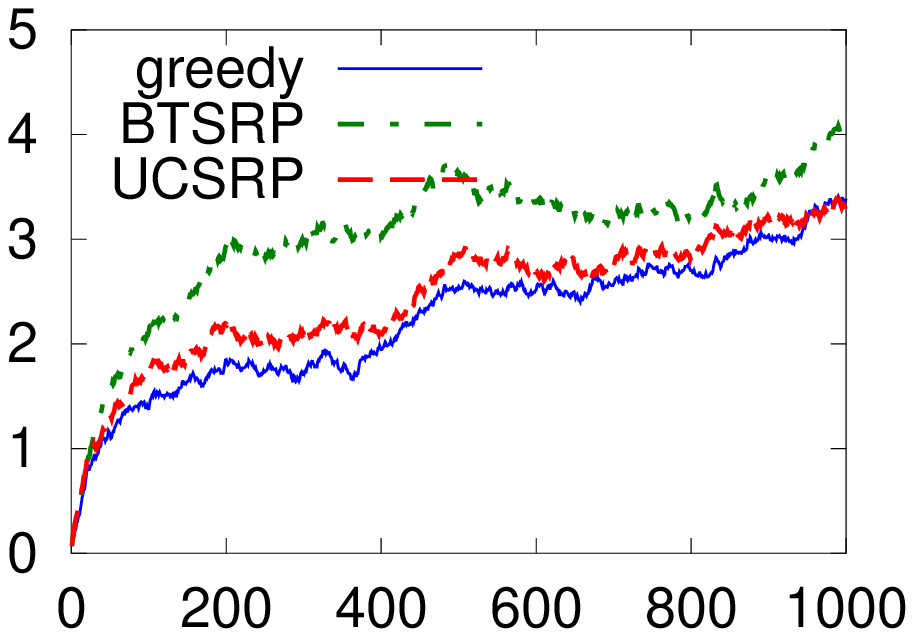}
    \label{fig:four-states-nature}
  }
  \subfigure[Nature, $|\CS| = 8, |\CA|=4, \term=0.1$]{
    \psfrag{greedy}[r][r][1.0][0]{greedy}
    \psfrag{BTSRP}[r][r][1.0][0]{BTSRP}
    \psfrag{UCSRP}[r][r][1.0][0]{UCSRP}
    \psfrag{0}[r][r][1.0][0]{0}
    \psfrag{100}[r][r][1.0][0]{200}
    \psfrag{200}[r][r][1.0][0]{200}
    \psfrag{300}[r][r][1.0][0]{300}
    \psfrag{400}[r][r][1.0][0]{400}
    \psfrag{500}[r][r][1.0][0]{500}
    \psfrag{600}[r][r][1.0][0]{600}
    \psfrag{800}[r][r][1.0][0]{800}
    \psfrag{1000}[r][r][1.0][0]{1000}
    \includegraphics[width=.475\textwidth]{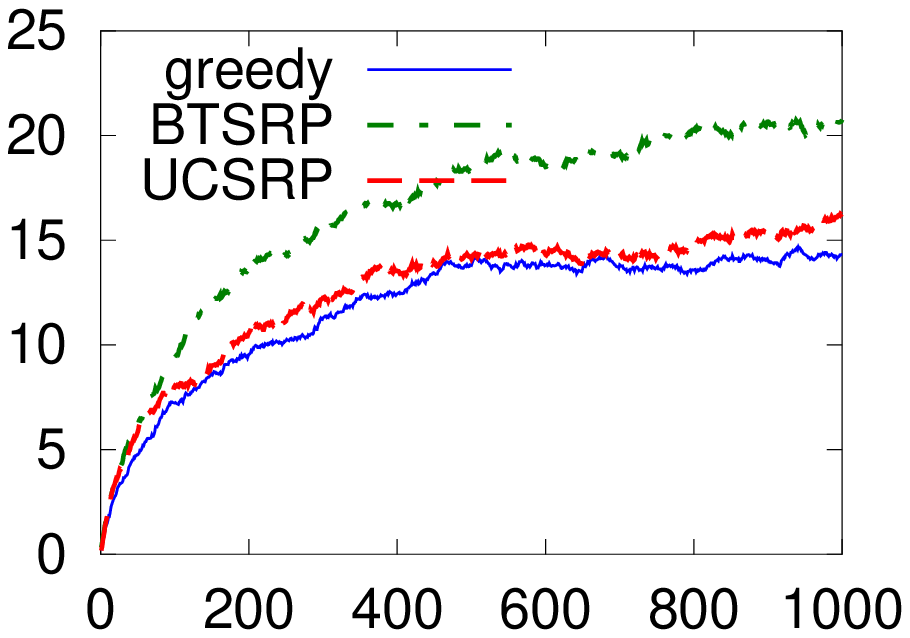}
    \label{fig:eight-states-nature}
  }
  \caption{Comparison of the expected cumulative regret after $k$
    stages, averaged over $10^3$ runs, between BTSRP and stagewise stationary greedy policies, on
    randomly generated MDPs. Against nature, none of the policies
    suffer a large amount of regret. Against an adversarial opponent,
    the greedy policy suffers linear regret. When the opponent is
    nature, the exploring policies enjoy no advantage.}
  \label{fig:markov-comparison}
\end{figure*}

\section{Discussion}
\label{sec:discussion}
We introduced {\em sparse reward processes}, which capture the problem of acting in an unknown environment with arbitrarily selected future objectives.  We have shown that, in an two-stage adversarial problem, a good strategy is to maximise the expected information gain. This links with previous work on curiosity and statistical decision theory. In fact, the connection of information gain to multiple tasks had been arguably recognised by~\citet{Lindley:InformationExperiment} 
\begin{quote}
  \ldots although indisputably one purpose
  of experimentation is to reach decisions, another purpose is to gain
  knowledge about the state of nature (that is, about the parameters)
  without having specific actions in mind.
\end{quote}
We have evaluated three algorithms on various problem instances. Overall, when the opponent is nature, even the greedy strategy performs relatively well. This is because it is forced to explore the environment by the sequence of payoffs. However, an adversarial opponent necessitates the use of the more sophisticated algorithms, which tend to explore the environment. This is partially explained by results in the related setting of multi-armed bandit problems with covariates~\cite{Sarkar:one-armed-bandit-covariates:1991,Yang:bandit-covariates:2002,pavlidis2008simulation:covariates,RigZee:bandits:covariates:10}. There, again the payoff function is given at the beginning of every stage. In that setting, however, the opponent is nature and, more importantly, the only thing observed after an action is chosen is a noisy reward signal. So, in some sense, it is a harder problem than the one considered herein (and indeed~\cite{RigZee:bandits:covariates:10} prove a lower bound). 
The one-armed covariate bandit~\cite{Sarkar:one-armed-bandit-covariates:1991} for an exponential family model, and proves that a myopic policy is asymptotically optimal, in a discounted setting. This ties in very well with our results on problems where the opponent is nature. 

Finally, SRPs are related to other multi-task learning settings. For example~\citep{lugosi:online-multi-task:colt:2008}, consider the problem of online multi-task learning with hard constraints. That is, at every round, the agent takes an action in each and every task, but there are some constraints which reflect the tasks' similarity. Somewhat closer to SRPs is the game-theoretic setting of \citet{mannor:geometric-multi-criterion:jmlr}, where again the agent is solving a multi-objective problem where the goal is that a reward vector approaches a target set. Finally, there is a close relation to the problem of learning with multiple bandits~\citep{dimitrakakis+lagoudakis:mlj2008,gabillon:mbbai}. Essentially, this problem involves finding near-optimal policies for a number of possibly related sub-problem within a search budget. In \citep{gabillon:mbbai} the tasks are unrelated bandit problems, while in \citep{dimitrakakis+lagoudakis:mlj2008} the tasks are actually different states of a Markov decision process and the goal is to find the best initial actions given a rollout policy.

Finally, our experimental results show that a greedy policy is a good strategy when the payoff sequence is (uniformly) stochastic. This naturally encourages exploration, even for non-curious agents, by forcing them to visit all states frequently.  UCSRP and BTSRP, which explore naturally, perform much better for adversarial payoffs. Then the greedy player suffers linear total regret. Consequently, we may conclude that curiosity is not in fact necessary when the constant change of goals forces exploration upon the agent.

\bibliographystyle{plainnat}
\bibliography{../../bib/misc,../../bib/mine}

\end{document}